\documentclass{article}

\usepackage{microtype}
\usepackage{graphicx}
\usepackage{subfigure}
\usepackage{booktabs} 
\usepackage{nicefrac}       
\usepackage{amsmath,amsthm,amssymb,amsfonts} 
\usepackage{mathtools}
\usepackage{balance}

\mathtoolsset{showonlyrefs}

\usepackage{hyperref}

\usepackage[accepted]{icml2019}

\icmltitlerunning{Surrogate Losses for Online Learning of Stepsizes in Stochastic Non-Convex Optimization}

\newtheorem{thm}{Theorem}
\newtheorem{lemma}[thm]{Lemma}

\theoremstyle{definition}

\newcommand{\field}[1]{\mathbb{#1}}

\newcommand{\R}{\field{R}}

\newcommand{\E}{\field{E}}

\DeclareMathOperator*{\argmin}{arg\,min}

\newcommand{\bg}{\boldsymbol{g}}

\newcommand{\bx}{\boldsymbol{x}}

\newcommand{\bu}{\boldsymbol{u}}
\newcommand{\bv}{\boldsymbol{v}}
\newcommand{\by}{\boldsymbol{y}}

\newcommand{\etab}{\boldsymbol{\eta}}

\begin{document}

\twocolumn[
\icmltitle{Surrogate Losses for Online Learning of Stepsizes in Stochastic Non-Convex Optimization}



\icmlsetsymbol{equal}{*}

\begin{icmlauthorlist}
\icmlauthor{Zhenxun Zhuang}{bucs}
\icmlauthor{Ashok Cutkosky}{goo}
\icmlauthor{Francesco Orabona}{bucs,buee}
\end{icmlauthorlist}

\icmlaffiliation{bucs}{Department of Computer Science, Boston University, Boston, MA, US}
\icmlaffiliation{goo}{Google, Mountain View, CA, US}
\icmlaffiliation{buee}{Department of Electrical \& Computer Engineering, Boston University, Boston, MA, US}

\icmlcorrespondingauthor{Zhenxun Zhuang}{zxzhuang@bu.edu}

\icmlkeywords{Online Learning, Stochastic Optimization, Non-convex}

\vskip 0.3in
]



\printAffiliationsAndNotice{}  

\

\begin{abstract}
Stochastic Gradient Descent (SGD) has played a central role in machine learning. However, it requires a carefully hand-picked stepsize for fast convergence, which is notoriously tedious and time-consuming to tune. Over the last several years, a plethora of adaptive gradient-based algorithms have emerged to ameliorate this problem.
In this paper, we propose new surrogate losses to cast the problem of learning the optimal stepsizes for the stochastic optimization of a non-convex smooth objective function onto an online convex optimization problem. This allows the use of no-regret online algorithms to compute optimal stepsizes on the fly. In turn, this results in a SGD algorithm with self-tuned stepsizes that guarantees convergence rates that are automatically adaptive to the level of noise.
\end{abstract}

\section{Introduction}
In recent years, Stochastic Gradient Descent (SGD) has become the tool of choice for fast optimization of convex and non-convex objective functions. Its simplicity of implementation allows for use in virtually any machine learning problem.
In its basic version, it iteratively updates the solution to an optimization problem, moving in the negative direction of the gradient of the objective function at the current solution:
\begin{equation}
\label{eq:sgd}
\bx_{t+1} = \bx_t - \eta_t \bg(\bx_t, \xi_t),
\end{equation}
where $\bg(\bx_t, \xi_t)$ is a stochastic gradient of the objective function $f$ at the current point $\bx_t$ depending on the stochastic variable $\xi_t$. A critical component of the algorithm is the stepsize $\eta_t>0$. In order to achieve a fast convergence, the stepsizes must be carefully selected, taking into account the objective function and characteristics of the noise. This task becomes particularly daunting because the noise might change over time due to a variety of factors such as, e.g., approaching the local minimum of the function, changing size of the minibatch, gradients calculated through a simulation.

For the above reasons, a number of variants of SGD have been proposed trying to ``adapt'' the stepsizes in more or less theoretically principled ways. Indeed, the idea of adapting stepsizes is an old one. A few famous examples are the Polyak's rule~\cite{Polyak87}, Stochastic Meta-Descent~\cite{Schraudolph99}, AdaGrad~\cite{DuchiHS11}. However, most of previous approaches to adapting the stepsizes are designed for convex functions or without a guaranteed strategy of converging to some optimal stepsize. In fact, often the definition itself of ``optimal'' stepsize is missing.

In this paper, we take a different and novel route. We study theoretically the setting of stochastic smooth non-convex optimization and we design \emph{convex surrogate loss functions that upper bound the expected decrement of the objective function after an SGD update}. The first advantage of our approach is that the optimal stepsize can be now defined as the one minimizing the surrogate losses. Moreover, using a no-regret online learning algorithm~\cite{Cesa-BianchiL06}, we can adapt the stepsizes and guarantee that they will be close to the one of the a-posteriori optimal stepsize. Moreover, basing our approach on online learning methods, we gain the implicit robustness of these methods to adversarial conditions.

The rest of the paper is organized as follows. We begin by discussing related work (Section~\ref{sec:rel}), and then introduce necessary definitions and assumptions (Section~\ref{sec:def}). Next, we introduce the surrogate loss functions (Section~\ref{sec:surrogate}) and use them to design an algorithm that adapts global and coordinate-wise stepsizes (Sections~\ref{sec:ftrl} and~\ref{sec:coordinate}). We also empirically validate our theoretical findings on a classification task (Section~\ref{sec:exp} and Appendix) showing that, differently from other adaptive methods, our method does not require fiddling with stepsizes to guarantee convergence in the stochastic setting. Finally, we draw some conclusions and describe the future work in Section~\ref{sec:conc}.

\section{Related Work}
\label{sec:rel}

Here we discuss the theoretical related work on adaptive stochastic optimization algorithms. First, the convergence of a random iterate of SGD for non-convex smooth functions has been proved by \citet{Ghadimi13}. They also calculate how the optimal stepsize depends on the variance of the noise in the gradients and the smoothness of the objective function.

The optimal convergence rate was also obtained by \citet{WardWB18} using AdaGrad global stepsizes, without the need to tune parameters. \citet{LiO19} improves over their results by removing the assumption of bounded gradients. However, both analyses focus on the adaptivity of non-per-coordinate updates, and are somewhat complicated in order to deal with unbounded gradients or non-independence of the current stepsize from the current step gradient. In comparison, our technique is relatively simple, allowing us to easily show a nontrivial guarantee for per-coordinate updates. In addition, their results cannot recover linear rates of convergence assuming, for example, the Polyak-\L{}ojasiewicz condition~\cite{KarimiNS16}.

The idea of tuning stepsizes with online learning has been explored in the online convex optimization literature~\cite{KoolenvEG14,vanErvenK16}. There, the possible stepsizes are discretized and an expert algorithm is used to select the stepsize to use online. Instead, in our work the use of convex surrogate loss functions allows us to directly learn the optimal stepsize, without needing to discretize the range of stepsizes. This becomes very important when we consider the possibility of learning a stepsize for each coordinate (Section~\ref{sec:coordinate}), as we avoid a computational overhead exponential in the dimension of the space $d$ that a discretization would incur.

\section{Definitions and Assumptions}
\label{sec:def}
We use bold lower-case letters to denote vectors, and bold upper-case letters for matrices, e.g., $\bu\in\R^d, \boldsymbol{A}\in\R^{m\times n}$. The i$^{th}$ coordinate of a vector $\bu$ is $u_i$.
Throughout this paper, we study the Euclidean space $\R^d$ with the inner product $\langle\cdot, \cdot\rangle$. Unless explicitly noted, all norms are the Euclidean norm.
The dual norm $\|\cdot\|_*$ is the norm defined by $\|\bv\|_*=\sup_{\bu}\{\langle \bu, \bv\rangle:\|\bu\|\le1\}$. 
$\E[\bu]$ means the expectation w.r.t. the underlying probability distribution of a random variable $\bu$.
The gradient of $f$ at $\bx$ is denoted by $\nabla f(\bx)$. 

Now, we describe our first-order stochastic black-box oracle. In our setting, we will query the stochastic oracle two times on each round $t=1,\ldots,T$, obtaining the noisy gradients $\bg(\bx_t,\xi_t)$ and $\bg(\bx_t,\xi'_t)$. \emph{Note that, when convenient, we will refer to $\bg(\bx_t,\xi_t)$ and $\bg(\bx_t,\xi'_t)$ as $\bg_t$ and $\bg_t'$ respectively.} We assume everywhere that our objective function is bounded from below and denote the infimum by $f^\star$, hence $f^\star> -\infty$. Also, we use $\E_t[\bu]$ to denote the conditional expectation of $\bu$ with respect to $\xi_1,\ldots,\xi_{t-1},\xi'_1,\ldots,\xi'_{t-1}$. Further, we will make use of the following assumptions:

\textbf{H1:}\quad The noisy gradients at time $t$ are unbiased and independent given the past, that is
\begin{align*}
&\E_{t}\left[\bg(\bx_t,\xi_t)\right]=\E_{t}\left[\bg(\bx_t,\xi'_t)\right]
=\nabla f(\bx_t),\\
&\E_{t}\left[\langle \bg(\bx_t,\xi_t),\bg(\bx_t,\xi'_t)\rangle\right]
=\|\nabla f(\bx_t)\|^2~.
\end{align*}
\textbf{H2:}\quad The noisy gradients $\bg_t$ have finite variance with respect to the L2 norm:
\begin{align*}
\E_{t}\left[\left\|\bg(\bx_t,\xi_t)-\nabla f(\bx_t)\right\|^2\right]
&=\sigma_t^2~.
\end{align*}
\textbf{H3:}\quad The noisy gradients have bounded norm: $\|\bg(\bx_t,\xi_t)\|\le L,\ \|\bg(\bx_t,\xi'_t)\|\le L$.

\textbf{H4:}\quad The function $f:\R^d\rightarrow\R$ is $M$-smooth, that is $f$ is differentiable and $\forall \bx_1,\bx_2\in\R^d$ we have $\|\nabla f(\bx_1)-\nabla f(\bx_2)\| \leq M \|\bx_1-\bx_2\|$. Note that (\textbf{H4}), for all $\bx,\by \in \R^d$, implies~\citep[Lemma 1.2.3]{Nesterov2003}
\begin{equation}
\label{eq:smooth2}
\left|f(\bx_2)-f(\bx_1)-\langle \nabla f(\bx_1), \bx_2-\bx_1\rangle\right|
\leq \frac{M}{2}\|\bx_2-\bx_1\|^2.
\end{equation}
We will also consider the Polyak-\L{}ojasiewicz (PL) condition~\cite{KarimiNS16}, a much weaker version of \emph{strong convexity}. The PL condition does not require convexity, but is still sufficient for showing a global linear convergence rate for gradient descent.

\textbf{H5:}\quad A differentiable function $f$ satisfies the PL condition if for some $\mu>0$
\begin{equation*}
\|\nabla f(\bx)\|^2\ge 2\mu(f(\bx) - f^\star),\quad \forall\bx~.
\end{equation*}

\section{Surrogate Losses}
\label{sec:surrogate}

Consider using SGD with non-convex $M$-smooth losses, using a fixed stepsize $0<\eta\leq\frac{1}{M}$ and starting from an initial point $\bx_1$. Assuming all the variances are bounded by $\sigma^2$, it is well-known that we obtain the following convergence rate~\cite{Ghadimi13}
\[
\E [\|\nabla f(\bx_k)\|^2]
\leq O\left( (f(\bx_1)-f^\star)/(\eta T) + \eta \sigma^2  \right),
\]
where $k$ is a uniform random variable between $1$ and $T$. From the above, it is immediate to see that we need a stepsize of the form $O(\min(\tfrac{\sqrt{f(\bx_1)-f^\star}}{\sigma \sqrt{T}},\tfrac{1}{M}))$ to have the best worst case convergence of $O(\tfrac{1}{T}+\tfrac{\sigma}{\sqrt{T}})$. In words, this means that we get a faster rate, $O(\tfrac{1}{T})$, when there is no noise, and a slower rate, $O(\tfrac{\sigma}{\sqrt{T}})$, in the presence of noise.

However, we usually do not know the variance of the noise $\sigma$, which makes the above optimal tuning of the stepsize difficult to achieve in practice. Even worse, the variance can change over time. For example, it may decrease over time if $f(\bx) = \E[f_j(\bx)]$ and each $f_j$ has zero gradient at the local optimum we are converging to. Moreover, even assuming the knowledge of the variance of the noise, the stepsizes proposed in \citet{Ghadimi13} assume the knowledge of the unknown quantity $f(\bx_1)-f^\star$.

One solution would be to obtain an explicit estimate of the variances of the noise, for example by applying some concentration inequality to the sample variance, and use it to set the stepsizes. This approach is suboptimal because it does not directly optimize the convergence rates, relying instead on a loose worst-case analysis. Instead, we propose to directly estimate the stepsizes that achieve the best convergence rate using an online learning algorithm. Our approach is simple and efficient: we introduce new surrogate (strongly)-convex losses that make the problem of learning the stepsizes a simple one-dimensional online convex optimization problem.

Our strategy uses the smoothness of the objective function to transform the problem of optimizing a non-convex objective function to the problem of optimizing a series of convex loss functions, which we solve by an online learning algorithm.
Specifically, at each time $t$ define the surrogate loss $\ell_t:\R \rightarrow\R$ as
\begin{equation}
\ell_t(\eta) 
=\frac{M \eta^2}{2} \|\bg(\bx_t,\xi_t)\|^2
-\eta \langle \bg(\bx_t,\xi_t), \bg(\bx_t,\xi'_t) \rangle, \label{eq:surrogate}
\end{equation}
where $\bg(\bx_t,\xi_t)$ and $\bg(\bx_t,\xi'_t)$ are the noisy stochastic gradients received from the black-box oracle at time $t$.
It is clear that $\ell_t$ is a convex function. Moreover, the following key result shows that these surrogate losses upper bound the expected decrease of the function value $f$.
\begin{thm}
Assume (\textbf{H4}) holds and $\eta_t$ is independent from $\xi_j$ and $\xi'_j$ for $j\geq t$. Then, for the SGD update in \eqref{eq:sgd}, we have
\label{thm:surrogate}
\[
\E\left[f(\bx_{t+1})-f(\bx_t)\right]
\leq \E\left[\ell_t(\eta_t)\right]~.
\]
\end{thm}
\begin{proof}
The $M$-smoothness of $f$ gives us:
\begin{align*}
&\E\left[f(\bx_{t+1})-f(\bx_t)\right]\\
&\le \E\left[\langle\nabla f(\bx_t),\bx_{t+1}-\bx_t \rangle + \frac M2\|\bx_{t+1}-\bx_t\|^2\right]\\
&= \E\left[\langle\nabla f(\bx_t),-\eta_t \bg(\bx_t,\xi_t) \rangle+\frac M2\eta_t^2\|\bg(\bx_t,\xi_t)\|^2\right]\\
&= \E\left[\langle\nabla f(\bx_t),\E_t\left[-\eta_t \bg(\bx_t,\xi_t)\right] \rangle+\frac M2\eta_t^2\|\bg(\bx_t,\xi_t)\|^2\right].
\end{align*}
Now observe that $\nabla f(\bx_t)=\E_t\left[\bg(\bx_t,\xi'_t)\right]$, so that
\begin{align*}
\E&\left[\langle \E_t\left[\bg(\bx_t,\xi'_t)\right],\E_t\left[-\eta_t \bg(\bx_t,\xi_t)\right]\rangle\right]\\
&= \E\left[ \E_t\left[\langle \bg(\bx_t,\xi'_t),-\eta_t \bg(\bx_t,\xi_t) \rangle\right]\right]\\
&= \E\left[\langle \bg(\bx_t,\xi'_t),-\eta_t \bg(\bx_t,\xi_t) \rangle\right]~.
\end{align*}
Putting all together, we have the stated inequality.
\end{proof}

Note that the assumption of the independence of $\eta_t$ from 
the stochasticity $\xi_j$ and $\xi'_j$ of the current step is essential according to~\citet{LiO19}.

The theorem tells us that if we want to decrease the function $f$, we might instead try to minimize the convex surrogate losses $\ell_t$. In the following, we build up on this intuition to design an online learning procedure that adapts the stepsizes of SGD to achieve the optimal convergence rate.

\section{SGDOL: Adaptive Stepsizes with FTRL}
\label{sec:ftrl}

\begin{algorithm}[tb]
\caption{Stochastic Gradient Descent with Online Learning (SGDOL)}
\label{algo:sgdol}
\begin{algorithmic}[1]
\STATE {\bfseries Input:} $\bx_1\in\mathcal{X},\ M$, an online learning algorithm $\mathcal{A}$
\FOR{$t = 1,2,\ldots,T$}
\STATE{\textbf{Compute} $\eta_t$ by running $\mathcal{A}$ on $\ell_{i}, i=1,\ldots,t-1$, as defined in \eqref{eq:surrogate}} \STATE{\textbf{Receive} two independent unbiased estimates of $\nabla f(\bx_{t})$: $\bg_t$, $\bg'_t$}
\STATE{\textbf{Update} $\bx_{t+1}=\bx_{t}-\eta_t \bg_t$}
\ENDFOR
\STATE{\textbf{Output}: uniformly randomly choose a $\bx_k$ from $\bx_1,\ldots,\bx_T$.}
\end{algorithmic}
\end{algorithm}

The surrogate losses allow us to design an online convex optimization procedure to learn the optimal stepsizes. We call this procedure Stochastic Gradient Descent with Online Learning (SGDOL) and the pseudocode is in Algorithm~\ref{algo:sgdol}. Remind that $\bg(\bx_t,\xi_t)$ and $\bg(\bx_t,\xi'_t)$ are referred to as $\bg_t$ and $\bg_t'$ respectively for convenience. In each round, the stepsizes are chosen by an online learning algorithm $\mathcal{A}$ fed with the surrogate losses $\ell_t$. The online learning algorithm aims at minimizing the regret: the difference between the cumulative sum of the losses incurred by the algorithm in each round, and the cumulative losses w.r.t any fixed point $\eta$ (especially the one giving the smallest losses in hindsight). In formulas, for a 1-dimensional online convex optimization problem, the regret is defined as
\[
\mathrm{Regret}_T(\eta) =\sum_{t=1}^T (\ell_t(\eta_t)- \ell_t(\eta))~.
\]
A small regret with respect to the optimal choice of $\eta$ means that the losses of the algorithm are not too big compared to the best achievable losses from using a fixed point.
In turn, this implies that the stepsizes chosen by the online algorithm will not be too far from the optimal (unknown) stepsize.

Employing SGDOL, we can prove the following Theorem.
\begin{thm}
\label{thm:olsmooth}
Assume \textbf{(H1, H2, H4)} to hold. Then, for any $\eta>0$, SGDOL in Algorithm~\ref{algo:sgdol} satisfies
\begin{align*}
\E&\left[\left(\eta-\frac{M}{2}\eta^2\right)\sum_{t=1}^T \|\nabla f(\bx_t)\|^2\right] \\
&\leq f(\bx_1) - f^\star + \E\left[\mathrm{Regret}_T(\eta)\right] + \frac{M \eta^2}{2}\sum^T_{t=1} \E[\sigma_t^2]~.
\end{align*}
\end{thm}
\begin{proof}
Summing the inequality in Theorem~\ref{thm:surrogate} from 1 to $T$:
\begin{align*}
f^\star-f(\bx_1)
\le&\ \E[f(\bx_{T+1})]-f(\bx_1)\\
=&\sum^T_{t=1}\E\left[f(\bx_{t+1})-f(\bx_t)\right]\\
\le&
\sum^T_{t=1}\E\left[\ell_t(\eta_t)\right]\\
=&
\sum^T_{t=1}\E\left[\ell_t(\eta_t)-\ell_t(\eta)\right]
+
\sum^T_{t=1}\E\left[\ell_t(\eta)\right]\\
\le&\ 
\E\left[\mathrm{Regret}_T(\eta)\right]
+
\sum^T_{t=1}\E\left[\ell_t(\eta)\right]~.
\end{align*}
Using the fact that
\begin{align*}
\E_t[\ell_t(\eta)] 
&= \left(-\eta+\frac{M}{2}\eta^2\right)\|\nabla f(\bx_t)\|^2 + \frac{M}{2}\eta^2\sigma_t^2,
\end{align*}
we have the stated bound.
\end{proof}

The only remaining ingredient for SGDOL is to decide an online learning procedure. Given that the surrogate losses are strongly convex, we can use a Follow The Regularized Leader (FTRL) algorithm~\cite{Shalev-Shwartz07,AbernethyHR08,AbernethyHR12,McMahan17}. Note that this is not the only possibility, e.g. we could even use an optimistic FTRL algorithm that achieves even smaller regret~\cite{mohri2016accelerating}. However, FTRL is enough to show the potential of our surrogate losses. 
In an online learning game in which we receive the convex losses $\ell_t$, FTRL constructs the predictions $\bv_t$ by solving the optimization problem 
\[
\bv_{t+1} = \argmin_{\bv\in\R^d} \ r(\bv) + \sum^t_{s=1} \ell_s(\bv),
\]
where $r:\R^d \rightarrow \R$ is a regularization function.
We can upper bound the regret of FTRL with the following theorem.

\begin{algorithm}[tb]
\caption{Follow the Regularized Leader (FTRL)}
\label{algo:aftrl}
\begin{algorithmic}[1]
\STATE{\textbf{Parameters}: $r(\bv)\geq 0$}
\STATE{$\bv_1\leftarrow\argmin_{\bv\in\R^d} \ r(\bv)$}
\FOR{$t = 1,2,\ldots$}
\STATE{Observe convex loss function $\ell_t:\R^d\rightarrow \R\cup\{\infty\}$}
\STATE{Incur loss $\ell_t(\bv_t)$}
\STATE{Update $\bv_{t+1}\leftarrow \argmin_{\bv\in\R^d} \ r(\bv) + \sum^t_{s=1} \ell_s(\bv)$}
\ENDFOR
\end{algorithmic}
\end{algorithm}


\begin{thm}{\cite{McMahan17}}
Suppose $r$ is chosen such that $h_{t} = r + \sum_{i=1}^t \ell_{i}$ is 1-strongly-convex w.r.t. some norm $\|\cdot\|_{(t)}$.
Then, choosing any $\bg_t\in\partial \ell_t(\bx_t)$ on each round, for any $\bx^\star\in\R^d$ and for any $T>0$,
\begin{equation}
\label{thm:prox}
\mathrm{Regret}_T(\bx^\star)
\le r(\bx^\star) + \frac12\sum^T_{t=1} \|\bg_t\|^2_{(t),\star},
\end{equation}
where $\|\cdot\|_{(t),\star}$ is the dual norm of $\|\cdot\|_{(t)}$.
\end{thm}

We can now put all together and obtain a convergence rate guarantee for SGDOL.
\begin{thm}
\label{thm:ftrlp}
By choosing $r(\eta)=\frac{M\alpha}2\left(\eta-\frac1M\right)^2+\mathcal{I}\left(\eta\in\left[0,\frac2M\right]\right)$ with $\alpha>0$, assuming \textnormal{(\textbf{H1 - H4})}, and using FTRL, Algorithm~\ref{algo:aftrl}, in Algorithm~\ref{algo:sgdol}, for an uniformly randomly picked $\bx_k$ from $\bx_1,\ldots,\bx_T$, we have:
\begin{align*}
\E&\left[\|\nabla f(\bx_k)\|^2\right]\\
&\le \frac{2M}{T}\left(f(\bx_1)-f^\star + \frac{5L^2}{M}\ln\left(1+\frac{L^2T}{\alpha}\right)\right)\\
&\quad + \frac1{T}\sqrt{2M\sum^T_{t=1} \E[\sigma_t^2]\left(f(\bx_1)-f^\star + \frac{\alpha}{2M}\right)}\\
&\quad + \frac1{T}\sqrt{10L^2\sum^T_{t=1} \E[\sigma_t^2]\ln\left(1+\frac{L^2T}{\alpha}\right)}~.
\end{align*}
\end{thm}

Before proving this theorem, we make some observations.

The FTRL update gives us a very simple strategy to calculate the stepsizes $\eta_t$. In particular, the FTRL update has a closed form:
\begin{align*}
\eta_{t} = \max\left\{0,\min\left\{\frac{\alpha+\sum^{t-1}_{j=1}\langle \bg_j, \bg'_j\rangle}{M\left(\alpha+\sum^{t-1}_{j=1}\|\bg_j\|^2\right)},\frac2M\right\}\right\}~.
\end{align*}
Note that this update can be efficiently computed by keeping track of the quantities $\sum^{t-1}_{j=1}\langle \bg_j, \bg'_j\rangle$ and $\sum^{t-1}_{j=1}\|\bg_j\|^2$.

While the computational complexity of calculating $\eta_t$ by FTRL is negligible, SGDOL requires two unbiased gradients per step. This increases the computational complexity with respect to a plain SGD procedure by a factor of two.

The value of $\alpha$ affects how fast $\eta_t$ deviates from its initial value $\frac1M$. Although Theorem~\ref{thm:ftrlp} shows that a too small $\alpha$ would blow up the log
factor, it also indicates setting $\alpha$ to be comparable with $M(f(x_ 1) - f(x^*))$ or smaller would suffice for not inducing a major influence on the convergence rate. In fact, preliminary experiments have shown that $\alpha$ has no notable influence on performance so long as it is comparable to M or smaller.

We can now prove the convergence rate in Theorem~\ref{thm:ftrlp}.
For the proof, we need the following technical lemma.
\begin{lemma}
\label{lm:integ}
Let $h:[0,+\infty)\rightarrow [0, +\infty)$ be a nonincreasing function, and $a_i\geq0$ for $i = 0, \cdots, T$.
Then
\begin{align*}
\sum_{t=1}^T a_t h\left(a_0+\sum_{i=1}^{t} a_i\right) 
&\leq \int_{a_0}^{\sum_{t=0}^T a_t} h(x) dx~.
\end{align*}
\end{lemma}
\begin{proof}
Denote by $s_t=\sum_{i=0}^{t} a_i$.
\begin{align*}
a_t h(s_t) 
=  \int_{s_{t-1}}^{s_t} h(s_t) d x 
\leq \int_{s_{t-1}}^{s_t} h(x) d x~.
\end{align*}
Summing over $t=1, \cdots, T$, we have the stated bound.
\end{proof}

\begin{proof}[Proof of Theorem~\ref{thm:ftrlp}]
As $\ell_t^{\prime\prime}(\eta)=M\|\bg_t\|^2$, $h_{t} = r + \sum_{i=1}^t \ell_{i}$ is 1-strongly-convex w.r.t.\ the norm $\sqrt{M\left(\alpha + \sum^t_{s=1}\|\bg_s\|^2\right)}\|\cdot\|$.


Applying Theorem~\ref{thm:prox}, we get that, for any $\eta\in\left[0,\frac2M\right]$,
\begin{align}
&\mathrm{Regret}_T(\eta) \nonumber \\
&\le \frac{M\alpha}2\left(\eta-\frac1M\right)^2 + \frac{1}{2M}\sum^T_{t=1}\frac{\left(\ell_t^{\prime}(\eta_t)\right)^2}{\alpha + \sum^t_{s=1}\|\bg_s\|^2}. \label{eq:regret}
\end{align}
Now observe that 
\begin{align*}
\left(\ell^{\prime}_t(\eta_t)\right)^2
&= \left(-\langle\bg_{t},\bg^{\prime}_{t} \rangle + M\eta_t\|\bg_{t}\|^2\right)^2\\
&\le 2\langle\bg_{t},\bg^{\prime}_{t} \rangle^2 + 2M^2\eta_t^2\|\bg_{t}\|^4\\
&\le 2\|\bg_{t}\|^2\|\bg_{t}^{\prime}\|^2 + 8\|\bg_{t}\|^4
\le 10L^2\|\bg_{t}\|^2,
\end{align*}
where in the third line of which we used the Cauchy-Schwarz inequality and $\eta_t\le\frac2M$. Hence, we get
\begin{align*}
\frac{1}{2M}\sum^T_{t=1}&\frac{\left(\ell_t^{\prime}(\eta_t)\right)^2}{\alpha + \sum^t_{s=1}\|\bg_s\|^2} \\
\leq& \frac{5L^2}{M}\sum^T_{t=1}\frac{\|\bg_t\|^2}{\alpha + \sum^t_{s=1}\|\bg_s\|^2}\\
\le& \frac{5L^2}{M}\ln\left(\frac{\alpha + \sum^T_{t=1}\|\bg_t\|^2}{\alpha}\right)\\
\le& \frac{5L^2}{M}\ln\left(1+\frac{L^2T}{\alpha}\right),
\end{align*}
where in the second inequality we used Lemma~\ref{lm:integ}.

Put the last inequality above back into Theorem~\ref{thm:olsmooth} yields
\begin{align*}
\E&\left[\left(\eta-\frac{M}{2}\eta^2\right)\sum_{t=1}^T \|\nabla f(\bx_t)\|^2\right] \\
&\leq f(\bx_1) - f^\star + \frac{M\alpha}2\left(\eta-\frac1M\right)^2 \\
&\quad + \frac{5L^2}{M}\ln\left(1+\frac{L^2T}{\alpha}\right) + \frac{M \eta^2}{2}\sum^T_{t=1} \E[\sigma_t^2]~.
\end{align*}
Denote $A\triangleq \sum^T_{t=1}\E\left[\|\nabla f(\bx_t)\|^2\right]$, we can transform the above into a quadratic inequality of $\eta$:
\begin{align*}
0\ \le&\phantom{+}\frac M2\left(A+\alpha+\sum^T_{t=1} \E[\sigma_t^2]\right)\eta^2 - (A+\alpha)\eta\\
&+ f(\bx_1) - f^\star + \frac{5L^2}{M}\ln\left(1+\frac{L^2T}{\alpha}\right) + \frac{\alpha}{2M}~.
\end{align*}
Choosing $\eta$ as the minimizer of the right hand side: $\eta^*=\frac{\alpha+A}{M(\alpha+A) + M\sum^T_{t=1} \E[\sigma_t^2]}$ (which satisfies $\eta^*\le\frac2M$) gives us
\begin{align*}
& \frac{(\alpha+A)^2}{2M\left(\alpha+A +\sum_{t=1}^T \E[\sigma_{t}^2]\right)}\\
\le\ &f(\bx_1)-f^\star + \frac{\alpha}{2M} + \frac{5L^2}{M}\ln\left(1+\frac{L^2T}{\alpha}\right)~.
\end{align*}
Solving this quadratic inequality of $A$ yields
\begin{align*}
A 
&\le 2M\left(f(\bx_1)-f^\star + \frac{5L^2}{M}\ln\left(1+\frac{L^2T}{\alpha}\right)\right)\\
&\quad + \sqrt{2M \sum^T_{t=1} \E[\sigma_t^2] \left(f(\bx_1)-f^\star + \frac{\alpha}{2M}\right)}\\
&\quad + \sqrt{10L^2 \sum^T_{t=1} \E[\sigma_t^2] \ln\left(1+\frac{L^2T}{\alpha}\right)}~.
\end{align*}

By taking an $\bx_k$ from $\bx_1,\ldots,\bx_T$ randomly, we get:
\begin{align*}
\E\left[\|\nabla f(\bx_k)\|^2\right] 
&= \E\left[\E\left[\|\nabla f(\bx_k)\|^2\middle\vert k\right]\right]\\
&= \frac1T\sum^T_{t=1}\E\left[\|\nabla f(\bx_t)\|^2\right],
\end{align*}
that completes the proof.
\end{proof}

\paragraph{Polyak-\L{}ojasiewicz Condition.}
When we assume in addition that the objective function satisfies the Polyak-\L{}ojasiewicz Condition~\cite{KarimiNS16} (\textbf{H5}), we can get the linear rate in the noiseless case.
\begin{thm}
\label{thm:pldet}
Choosing $r(\eta)=\frac{M\alpha}2\left(\eta-\frac1M\right)^2+\mathcal{I}\left(\eta\in\left[0,\frac2M\right]\right)$ with $\alpha>0$, assume \textnormal{(\textbf{H1 - H5})}, and that $\bg(\bx_t, \xi_t)=\bg(\bx_t, \xi'_t)=\nabla f(\bx_t)$ for all $t$ (i.e. there is no noise). Then feeding Algorithm~\ref{algo:aftrl} into Algorithm~\ref{algo:sgdol}, yields:
\begin{equation*}
\E\left[f(\bx_{T+1})-f^\star\right]
\le
\left(1-\frac{\mu}{M}\right)^T [f(\bx_1)-f^\star]~.
\end{equation*}
\end{thm}
\begin{proof}
From the update rule of $\eta_t$, we have that when there is no noise, $\eta_t = \frac1M$ all the time, thus:
\begin{align*}
\E&[f(\bx_{T+1})-f^\star]\\
&\le\E\left[f(\bx_{T})-f^\star+\langle\nabla f(\bx_T),\bx_{T+1}-\bx_T \rangle\right]\\
& \quad + \E\left[\frac M2\|\bx_{T+1}-\bx_T\|^2\right]\\
&=\E\left[f(\bx_{T})-f^\star + \left(-\eta_T + \frac{M\eta_T^2}2\right)\|\nabla f(\bx_T)\|^2\right]\\
&=\E\left[f(\bx_{T})-f^\star - \frac1{2M}\|\nabla f(\bx_T)\|^2\right]\\
&\le\E\left[\left(1-\frac{\mu}{M}\right)(f(\bx_{T})-f^\star)\right]
\le\cdots\le\\
&\le\left(1-\frac{\mu}{M}\right)^T[f(\bx_{1})-f^\star]~. \qedhere
\end{align*}
\end{proof}

\section{Adapting Per-coordinate Stepsizes}
\label{sec:coordinate}

In the previous Section, we have shown how to use the surrogate loss functions to adapt a stepsize. Another common strategy in practice is to use a \emph{per-coordinate} stepsize. This kind of scheme is easily incorporated into our framework and we show that it can provide improved adaptivity to per-coordinate variances.

Specifically, we consider $\etab_t$ now to be a vector in $\R^d$, $\etab_t=(\eta_{t,1},\dots,\eta_{t,d})$ and use the update $\bx_{t+1} = \bx_t - \etab_t\bg_t$
where $\etab_t \bg_t$ now indicates coordinate-wise product $(\eta_{t,1}g_{t,1},\dots,\eta_{t,d}g_{t,d})$. Then we define
\begin{align*}
\ell_t(\etab) &= \frac{M}{2}\|\etab \bg(\bx_t,\xi_t)\|^2
-
\langle \etab \bg(\bx_t,\xi_t), \bg(\bx_t,\xi_t')\rangle\\
&= \sum_{i=1}^d\left[ \frac{M}{2}\eta_{i}^2g(\bx_t,\xi_t)_i^2 - \eta_{i}g(\bx_t,\xi_t)_ig(\bx_t,\xi_t')_i\right]~.
\end{align*}
To take advantage of this scenario, we need more detail about the variance, which we encapsulate in the following assumption:

\textbf{H2':}\quad The noisy gradients $\bg_t$ have finite variance in each coordinate:
\begin{align*}
\E_{t}\left[(g(\bx_t,\xi_t)_i-\nabla f(\bx_t)_i)^2\right]
&=\sigma_{t,i}^2~.
\end{align*}
Note that this assumption is not actually stronger than (\textbf{H2}) because we can define $\sigma_t^2 = \sum_{i=1}^d \sigma_{t,i}^2$. This merely provides finer-grained variable names.

Also, we make the assumption:

\textbf{H3':}\quad The noisy gradients have bounded coordinate values: 
\begin{align*}
|g(\bx_t,\xi_t)_i|\le L_i,\ |g(\bx_t,\xi'_t)_i|\le L_i~.
\end{align*}

Now the exact same argument as for Theorem \ref{thm:olsmooth} yields:
\begin{thm}
\label{thm:olsmoothdiag}
Assume (\textbf{H4}) and the two noisy gradients in each round $t$ to satisfy \textbf{(H1)} and \textbf{(H2')}. Then, for any $\eta\in \R^d$ with $\eta_i>0$ for all $i$, the per-coordinate variant of Algorithm \ref{algo:sgdol} obtains
\begin{align*}
\E&\left[\sum_{t=1}^T \sum_{i=1}^d\left(\eta_i-\frac{M}{2}\eta_i^2\right)\nabla f(\bx_t)_i^2\right] \\
&\leq f(\bx_1) - f^\star + \E\left[\mathrm{Regret}_T(\etab)\right] + \frac{M}{2}\sum_{i=1}^d\sum_{t=1}^T \eta_i^2 \E[\sigma_{t,i}^2]~.
\end{align*}
\end{thm}

With this Theorem in hand, once again all that remains is to choose the online learning algorithm. To this end, observe that we can write $\ell_t(\eta) = \sum_{i=1}^d \ell_{t,i}(\eta_i)$ where
\begin{align*}
\ell_{t,i}(\eta_i) = \frac{M}{2}\eta_{i}^2g(x_t,\xi_t)_i^2 - \eta_{i}g(x_t,\xi_t)_ig(x_t,\xi_t')_i~.
\end{align*}
Thus, we can take our online learning algorithm to be a per-coordinate instantiation of Algorithm \ref{algo:aftrl}, and the total regret is simply the sum of the per-coordinate regrets. Each per-coordinate regret can be analyzed in exactly the same way as Algorithm \ref{algo:aftrl}, leading to
\begin{align*}
\mathrm{Regret}_T(\etab)&=\sum_{i=1}^d \mathrm{Regret}_{T,i}(\eta_i),\\
\mathrm{Regret}_{T,i}(\eta_i)&\le \frac{M\alpha}{2}\left(\eta_i-\frac{1}{M}\right)^2 + \frac{5L_i^2}{M}\ln\left(1+\frac{L_i^2T}{\alpha}\right)~.
\end{align*}
From these inequalities we can make a per-coordinate bound on the gradient magnitudes. In words, the coordinates which have smaller variances $\sigma^2_{t,i}$ achieve smaller gradient values faster than coordinates with larger variances. Further, we preserve adaptivity to the full variance $\sum^T_{t=1} \E[\sigma_t^2]$ in the rate of decrease of $\|\nabla f(x)\|^2$.
\begin{thm}\label{thm:pcftrlp}
Assume \textnormal{(\textbf{H1}, \textbf{H2'}, \textbf{H3'}, \textbf{H4})}. Suppose we run a per-coordinate variant of Algorithm~\ref{algo:sgdol}, with regularizer $r(\eta_i)=\frac{M\alpha}2\left(\eta_i-\frac1M\right)^2+\mathcal{I}\left(\eta_i\in\left[0,\frac2M\right]\right)$ in each coordinate with $\alpha>0$. Then, for each $i\in\{1,\dots,d\}$, we have
\begin{align*}
\E&\left[\sum_{t=1}^T\nabla f(\bx_t)_i^2\right]\\
&\le 2M\left(f(\bx_1)-f^\star + \sum_{i=1}^d\frac{5 L_i^2}{M}\ln\left(1+\frac{L_i^2T}{\alpha}\right)\right)\\
&\quad + \sqrt{2M\sum^T_{t=1} \E[\sigma_{t,i}^2]\left(f(\bx_1)-f^\star + \frac{d\alpha}{2M}\right)}\\
&\quad + \sqrt{10\sum^T_{t=1} \E[\sigma_{t,i}^2]\sum_{i=1}^d L_i^2\ln\left(1+\frac{L_i^2T}{\alpha}\right)}\\
&\quad+(d-1)\alpha~.
\end{align*}
Further, with $\sigma_t=\sum_{t=1}^T \sigma_{t,i}^2$ it also holds
\begin{align*}
\E&\left[\sum_{t=1}^T\|\nabla f(\bx_t)\|^2\right]\\
&\le 2M\left(f(\bx_1)-f^\star + \sum_{i=1}^d\frac{5 L_i^2}{M}\ln\left(1+\frac{L_i^2T}{\alpha}\right)\right)\\
&\quad + \sqrt{2M\sum^T_{t=1} \E[\sigma_t^2]\left(f(\bx_1)-f^\star + \frac{d\alpha}{2M}\right)}\\
&\quad + \sqrt{10\sum^T_{t=1} \E[\sigma_t^2]\sum_{i=1}^d L_i^2\ln\left(1+\frac{L_i^2T}{\alpha}\right)}~.
\end{align*}
\end{thm}

\begin{proof}
The proof is nearly identical to that of Theorem \ref{thm:ftrlp}. We have
\begin{align*}
\E&\left[\sum_{i=1}^d \left(\eta_i - \frac{M}{2}\eta_i^2\right)\sum_{t=1}^T \nabla f(\bx_t)_i^2\right]\\
&\le f(\bx_1)-f^\star + \frac{M\alpha }{2} \sum_{i=1}^d \left(\eta_i - \frac{1}{M}\right)^2\\
&+\sum_{i=1}^d \frac{5L_i^2}{M}\ln\left(1+\frac{L_i^2T}{\alpha}\right) +\sum_{i=1}^d \frac{M\eta_i^2}{2}\sum_{t=1}^T \E[\sigma_{t,i}^2]~.
\end{align*}

Let $A_i\triangleq\E[\sum_{t=1}^T \nabla f(\bx_t)_i^2]$ and choosing $\eta_i$ by the same strategy in Theorem~\ref{thm:ftrlp} as $\frac{\alpha+A_i}{M(\alpha+A_i+\sum_{t=1}^T \E[\sigma_{t,i}^2])}$ to obtain
\begin{align*}
\sum_{i=1}^d& \frac{(\alpha+A_i)^2}{2M\left(\alpha+A_i +\sum_{t=1}^T \E[\sigma_{t,i}^2]\right)}\\
&\le f(\bx_1)-f^\star + \frac{d\alpha}{2M}+\sum_{i=1}^d \frac{5L_i^2}{M}\ln\left(1+\frac{L_i^2T}{\alpha}\right)~.
\end{align*}
Now, the first statement of the Theorem follows by observing that each term on the LHS is non-negative so that the sum can be lower-bounded by any individual term. For the second statement, define
\begin{align*}
Q_i&=\frac{(\alpha+A_i)^2}{2M\left(\alpha+A_i +\sum_{t=1}^T \E[\sigma_{t,i}^2]\right)}\\
Q&=f(\bx_1)-f^\star + \frac{d\alpha}{2M}+\sum_{i=1}^d \frac{5L_i^2}{M}\ln\left(1+\frac{L_i^2T}{\alpha}\right),
\end{align*}
so that $\sum_{i=1}^d Q_i \le Q$. By the quadratic formula and definition of $Q_i$, we have
\begin{align*}
A_i &\le 2MQ_i + \sqrt{2MQ_i\sum_{t=1}^T \E[\sigma_{t,i}^2]} - \alpha~.
\end{align*}
Thus,
\begin{align*}
\sum_{i=1}^d A_i&\le 2MQ -d\alpha + \sum_{i=1}^d \sqrt{2MQ_i\sum_{t=1}^T \E[\sigma_{t,i}^2]}\\
&\le 2MQ -d\alpha + \sqrt{2M}\sqrt{\sum_{i=1}^d  Q_i}\sqrt{\sum_{i=1}^d\sum_{t=1}^T\E[\sigma_{t,i}^2]}\\
&=2MQ -d\alpha +  \sqrt{2MQ\sum_{t=1}^T \E[\sigma_t^2]}~.
\end{align*}
From which the second statement follows.
\end{proof}

\section{Experiments}
\label{sec:exp}

\begin{figure*}[t]
\begin{center}
\center
\includegraphics[width=\textwidth]{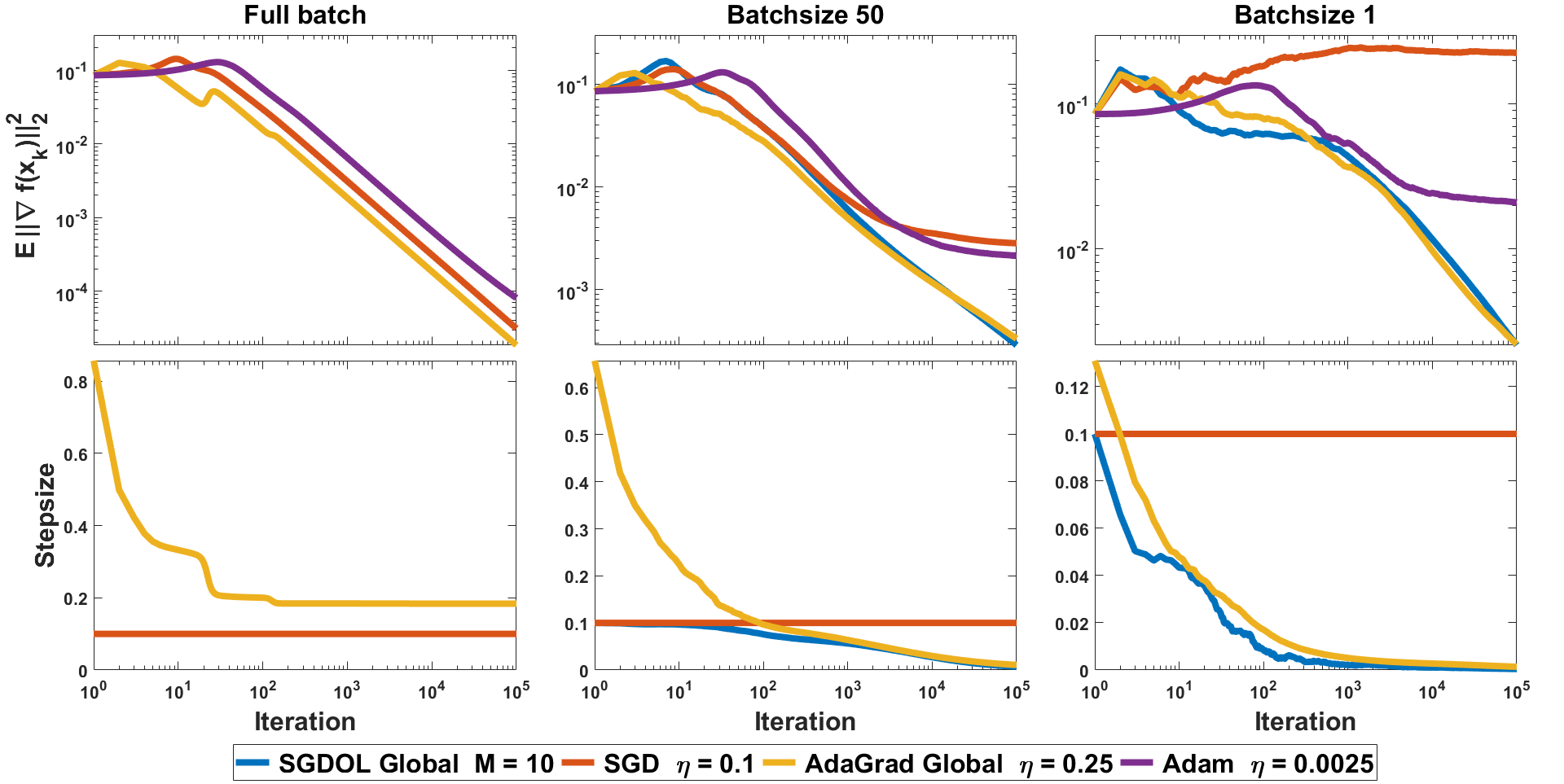}
\vspace{-1em}
\caption{Comparison of performance of SGDOL Global, SGD, AdaGrad Global, Adam on a non-linear classification model with different minibatch sizes.}
\label{fig:adult}
\end{center}
\vskip -0.2em
\end{figure*}

SGD is widely known to enjoy good empirical properties, but our learning rate schedule is very unique, so to validate our theoretical findings, we experiment on fitting a classification model on the adult (a9a) dataset from the LibSVM website~\cite{ChangL01}. The objective function is
\begin{equation*}
f(\bx):=\frac1m \sum^m_{i=1}\phi(\boldsymbol{a}_i^\top\bx-y_i),
\end{equation*}
where $\phi(\theta)=\frac{\theta^2}{1+\theta^2}$, and $(\boldsymbol{a}_i,y_i)$ are the couples feature vector/label. The loss function $\phi$ is non-convex, 1-Lipschitz and 2-smooth w.r.t. the $\ell_2$ norm.

We consider the minimization problem with respect to all training samples. Also, as the dataset is imbalanced towards the group with annual income less than 50K, we subsample that group to balance the dataset, which results in 15682 samples with 123 features each. In addition, we append a constant element to each sample feature vector to introduce a constant bias. $\bx_1$ is initialized to be all zeros. For each setting, we repeat the experiment independently but with the same initialization for 5 times, and plot the average of the relevant quantities.
In this setting, the noise on the gradient is generated by the use of minibatches.

We compare SGDOL with AdaGrad~\cite{DuchiHS10}, SGD, and Adam~\cite{KingmaB15} on three different minibatch sizes, namely different noise scales: using all samples, 50 i.i.d. samples, or 1 random sample for evaluating the gradient at a point. (\emph{Note that we adopt the scheme of using a single learning rate for all dimensions in SGDOL and AdaGrad thus the suffix `Global'.}) The learning rates of each algorithm, except for SGDOL Global, are selected as the ones giving the best convergence rate when the full batch scheme, namely zero noise, is employed, and are shown in the legend. We take the reciprocal of SGD's best learning rate as the parameter $M$ for SGDOL Global, and we set $\alpha=10$ without any tuning based on our discussion on the influence of $\alpha$ in Section~\ref{sec:ftrl}. These parameters are then employed in other two noisy settings.

We report the results in Figure~\ref{fig:adult}. In each column, the top plot shows $\E[\|\nabla f(\bx_k)\|^2]$ vs. number of iterations, whereas the bottom one is the per-round stepsizes on each case. Note that there is no per-round stepsize for Adam. The x-axis in all figures, and the y-axis in the top three are logarithmic. 

As can be seen, the stepsize of SGDOL is the same as SGD at first, but gradually decreases automatically. Also, the larger the noise, the sooner the decreasing phase starts. The decrease of the learning rate makes the convergence of SGDOL possible. In particular, SGDOL recovers the performance of SGD in the noiseless case, while it allows convergence in the noisy cases through an automatic decrease of the stepsizes. AdaGrad also enjoys nice convergence, and is comparable to ours. In contrast, when noise exists, after reaching a proximity of a stationary point, SGD and Adam oscillates thereafter without converging, and the value it oscillates around depends on the variance of the noise. This underlines the superiority of the surrogate losses, rather than choosing a stepsize based on a worst-case convergence rate bound.

More experiments can be found in the Appendix.

\section{Conclusions and Future Work}
\label{sec:conc}
We have presented a novel way to cast the problem of adaptive stepsize selection for the stochastic optimization of smooth (non-convex) functions as an online convex optimization problem with a simple quadratic convex surrogate. The reduction goes through the use of novel surrogate convex losses. This framework allows us to import the rich literature of no-regret online algorithms to learn stepsizes on the fly. We exemplified the power of this method with the SGDOL algorithm which enjoys an optimal convergence guarantee for any level of noise, without the need to estimate the noise nor tune the stepsizes. Moreover, we recover linear convergence rates under the PL-condition. The overall price to pay is a factor of 2 in the computation of the gradients. We also presented a per-coordinate version of SGDOL that achieves faster convergence on the coordinates with less noise.

We feel that we have barely scratched the surface of what might be possible with these surrogate losses. Hence, future work will focus on extending their use to other scenarios. For example, we plan to use it in locally private SGD algorithms where additional noise is added on the gradients to ensure privacy of the data~\cite{SongCS13}. We are also interested in investigating whether adding convexity would give us better results to recover SGD's performance. Another potential direction is to eliminate the need of knowing $M$, e.g. by automatically adapting to it on the fly. 

\section{Acknowledgment}
This material is based upon work supported
by the National Science Foundation under grant no.
1740762 ``Collaborative Research: TRIPODS Institute
for Optimization and Learning''.

\newpage
\balance

\bibliographystyle{icml2019}
\bibliography{../../../../learning}

\onecolumn
\section{Appendix}
\label{sec:append}

\subsection{2D Rosenbrock Function}
The popular 2-D Rosenbrock benchmark~\cite{Rosenbrock60} takes the form:
\begin{equation*}
f(x, y):=(1-x)^2 + 100(y-x^2)^2~.
\end{equation*}
It is non-convex and has one global minimum at $x=y=1$.

To add stochasticity, we apply additive white Gaussian noise to each gradient. To have a robust estimate of the performance, we repeat each experiment independently with the same parameters for 40 times and take the average.

We compare the performance of SGDOL Global with a bunch of popular adaptive optimization algorithms on the Rosenbrock function with 3 levels of added noise: zero noise, small noise ($\sigma=0.2$), and large noise ($\sigma=5$). The competitors are: SGD, AdaGrad Global~\cite{DuchiHS10} (with one learning rate for all dimensions), AdaGrad Coordinate~\cite{DuchiHS10} (with one learning rate for each dimension), Adam~\cite{KingmaB15}, RMSProp~\cite{TielemanH12}, AdaDelta~\cite{Zeiler12}, AMSGrad~\cite{ReddiKK18}, and Hypergradient~\cite{baydin-2018-hypergradient}. Also, we test the performance of the stepsize proposed by \citet{Ghadimi13}, denoted by SGD GL, given that in this synthetic experiment we know all the relevant quantities. We stress the fact that in the real-world setting this kind of stepsize cannot be used. We select the stepsize of all optimization algorithms except for SGDOL Global and SGD GL to be the one giving best convergence rate when running on the objective function with zero noise added. We choose $M$ of SGDOL to be the reciprocal of SGD's best learning rate which happens to be very close to the smoothness at the optimal point, 1002. We set $x_1=y_1=0$.

In Figure~\ref{fig:rosenmain}, the top plots show $\E[\|\nabla f(\bx_k)\|^2]$ vs. number of iterations, the middle ones reflect the curve of the optimality gap $f(\bx_t) - f^*$ at each round $t$ since we know $f^*=0$, whereas the bottom ones are the per-round stepsizes on each case. Note that there is no per-round stepsize for Adam. In Figure~\ref{fig:rosenother}, the top plots show $\E[\|\nabla f(\bx_k)\|^2]$ vs. number of iterations, and the bottom ones reflect the curve of the optimality gap $f(\bx_t) - f^*$ at each round $t$. The x-axis in all figures are logarithmic, and the y-axis in all figures except for those showing stepsizes are logarithmic.

The behavior of both the curve of $\E[\|\nabla f(\bx_k)\|^2]$ vs. number of iterations and the curve of stepsizes are similar to Figure~\ref{fig:adult}. And the behavior of the curve of optimality gap is similar to that of the curve of $\E[\|\nabla f(\bx_k)\|^2]$ vs. number of iterations.

\newpage
\begin{figure}[h]
\begin{center}
\center
\includegraphics[width=\columnwidth]{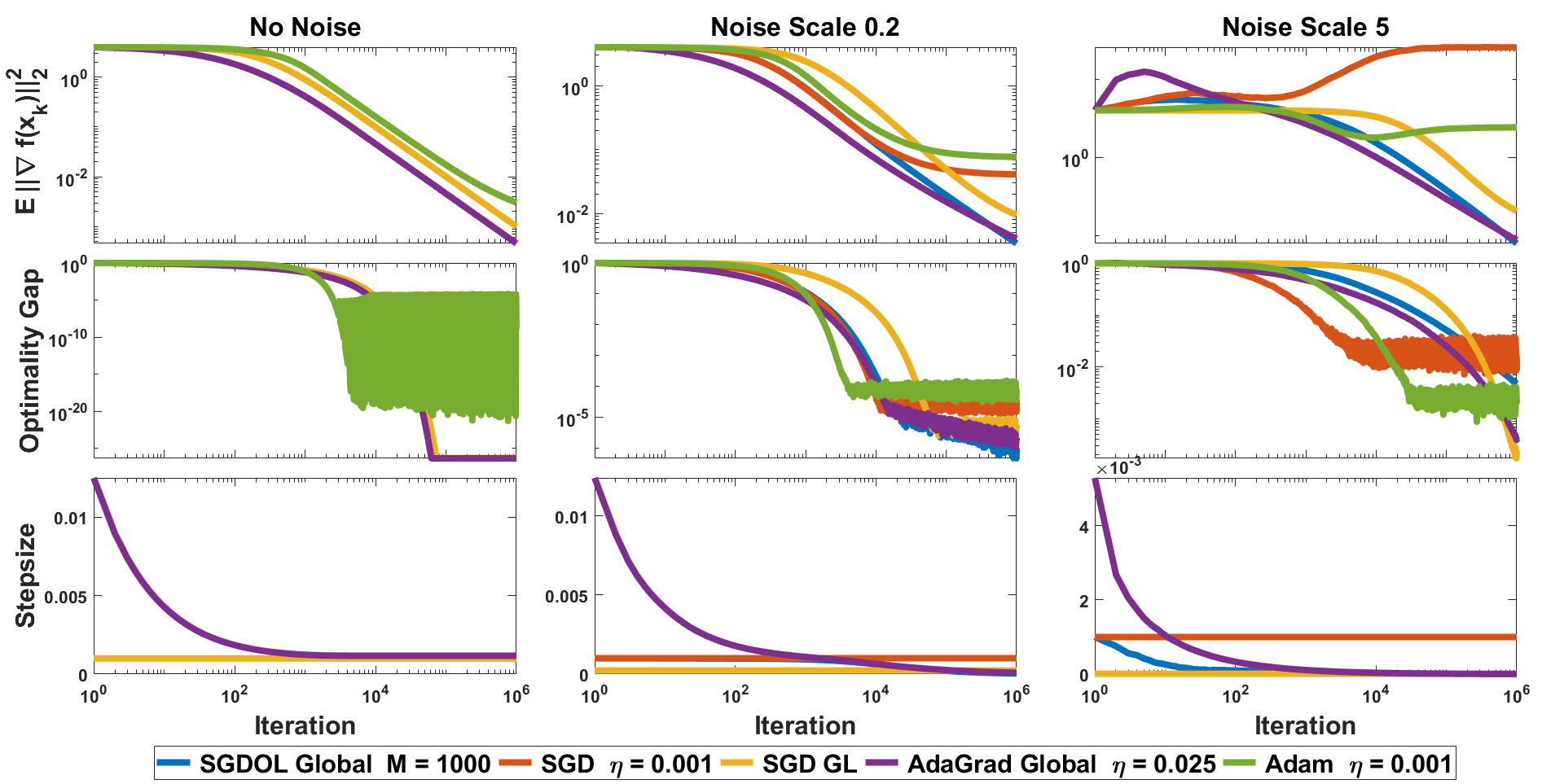}
\vspace{-1em}
\caption{Comparison of SGDOL Global, SGD, SGD GL, AdaGrad Global, and Adam on the 2D Rosenbrock function with different levels of noise.}
\label{fig:rosenmain}
\end{center}
\end{figure}

\begin{figure}[h!]
\begin{center}
\center
\includegraphics[width=\columnwidth]{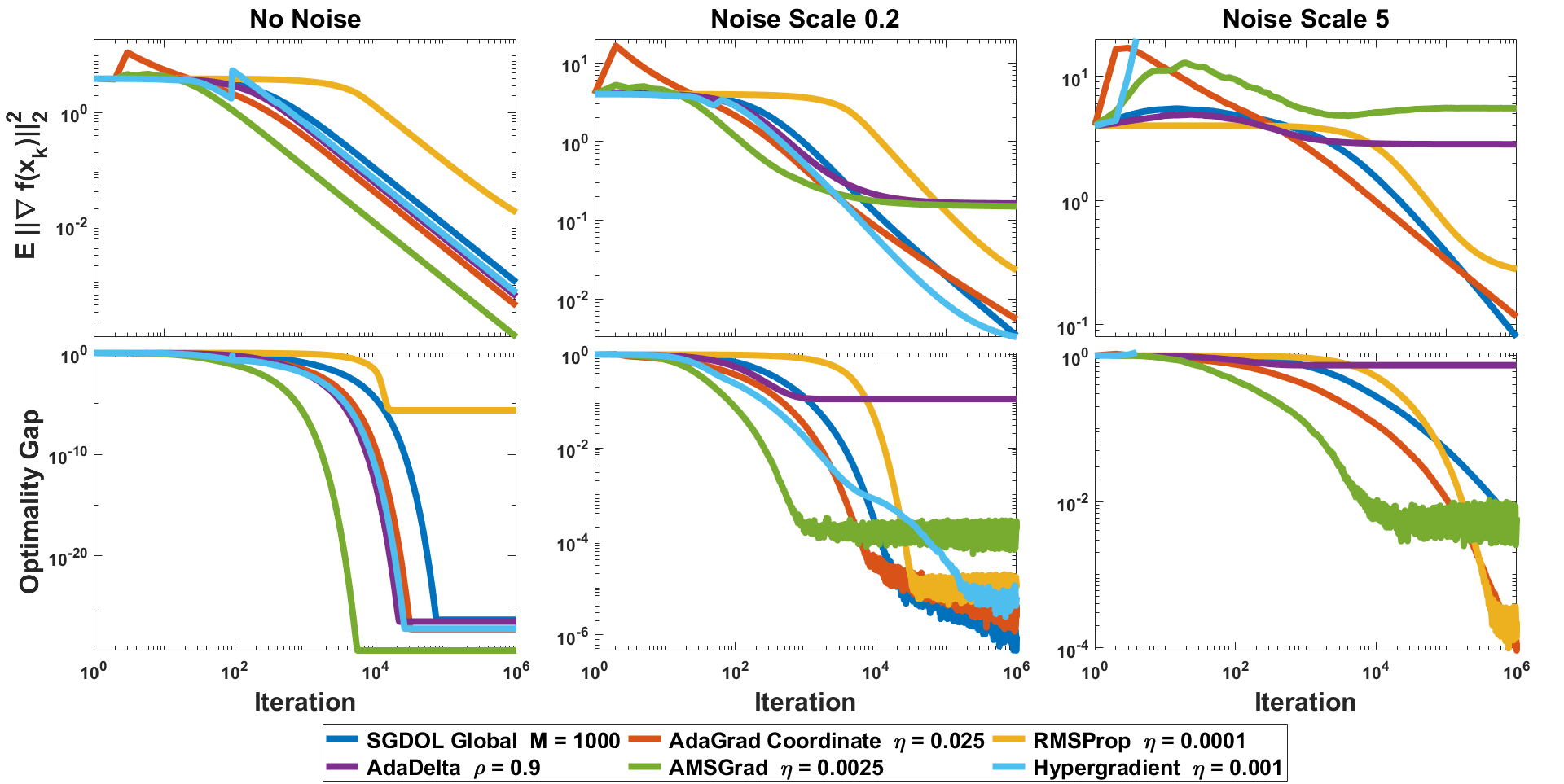}
\vspace{-1em}
\caption{Comparison of SGDOL Global, AdaGrad Coordinate, RMSProp, AdaDelta, AMSGrad, and Hypergradient on the 2D Rosenbrock function with different levels of noise.}
\label{fig:rosenother}
\end{center}
\end{figure}
\newpage
\subsection{Other Results for Fitting a Non-Linear Classification Model}
Here we show the results for comparison between SGDOL Global and other optimization algorithms listed in the above subsection but applied to the classification task introduced in Section~\ref{sec:exp}. Note that here we don't know $f^*$ so we don't report the curve of the optimality gap. The comparison shown in Figure~\ref{fig:adultother} is similar to what is reported in Figure~\ref{fig:adult}.
\begin{figure}[h!]
\begin{center}
\center
\includegraphics[width=\columnwidth]{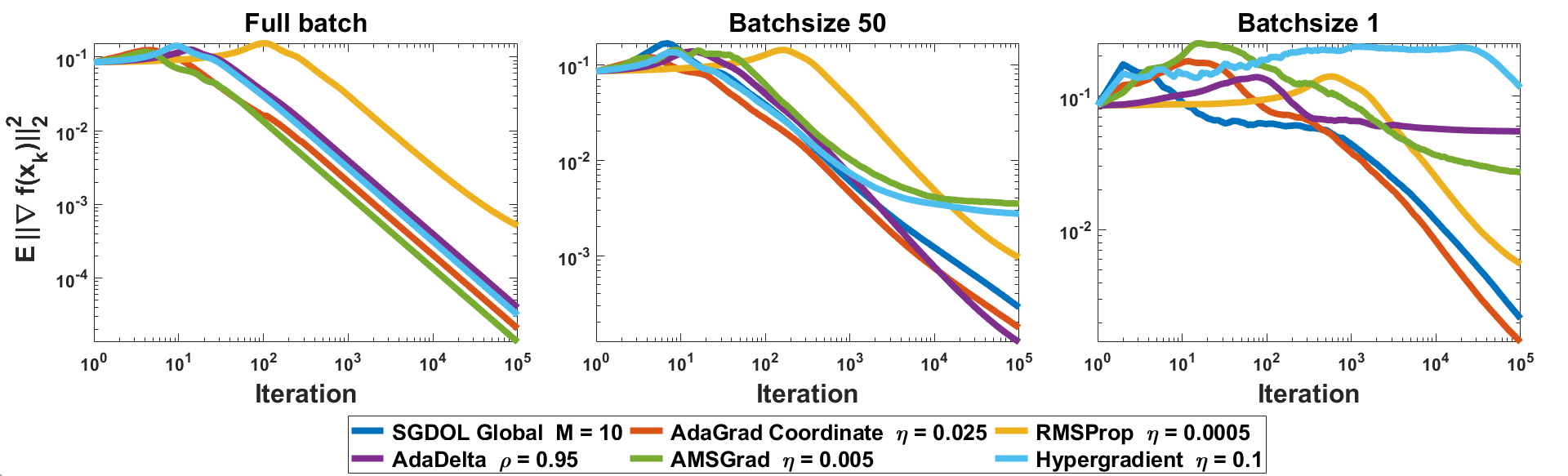}
\vspace{-1em}
\caption{Comparison of performance of SGDOL Global, AdaGrad Coordinate, RMSProp, AdaDelta, AMSGrad, and Hypergradient on a non-linear classification model with different minibatch sizes.}
\label{fig:adultother}
\end{center}
\end{figure}

\end{document}